\documentclass[sigconf,nonacm]{acmart}
\AtBeginDocument{%
  }

\setcopyright{acmlicensed}
\copyrightyear{2018}
\acmYear{2018}
\acmDOI{XXXXXXX.XXXXXXX}
\acmConference[Conference acronym 'XX]{Make sure to enter the correct
  conference title from your rights confirmation email}{June 03--05,
  2018}{Woodstock, NY}
\acmISBN{978-1-4503-XXXX-X/2018/06}

\usepackage{booktabs}
\usepackage{multirow}

\usepackage{amsmath,amssymb}
\usepackage{graphicx}
\usepackage{subcaption}
\usepackage{url}
\usepackage{hyperref}
\usepackage{xcolor}
\usepackage{enumitem}
\usepackage{cleveref}
\usepackage[colorinlistoftodos]{todonotes}
\usepackage[table]{xcolor}
\usepackage[ruled,vlined,linesnumbered]{algorithm2e}
\usepackage{enumitem}

\usepackage{tikz}
\usetikzlibrary{
    arrows.meta,
    positioning,
    shapes.geometric
}

\usepackage{graphicx}
\usepackage{float}

\newtheorem{assumption}{Assumption}

\makeatletter
\renewcommand{\@mktitle@iii}{\hsize=\textwidth
    \setbox\mktitle@bx=\vbox{\@titlefont\centering
      \@ACM@title@width=\hsize
      \parbox[t]{\@ACM@title@width}{\centering
        {\normalfont\itshape\normalsize Preprint. Under review.\par}%
        \vskip 0.8em%
        \@titlefont
        \@title\@translatedtitle%
        \ifx\@subtitle\@empty\else
          \par\noindent{\@subtitlefont\@subtitle\@translatedsubtitle}%
        \fi
      }%
      \par\bigskip}}%
\makeatother

\begin{document}

\title{Towards Systematic Generalization for Power Grid Optimization Problems}

\author{Zeeshan Memon}
\email{zeeshan.memon@emory.edu}
\affiliation{%
  \institution{Emory University}
  \city{Atlanta}
  \state{GA}
  \country{USA}
}

\author{Yijiang Li}
\email{yijiang.li@anl.gov}
\affiliation{%
  \institution{Argonne National Laboratory}
  \city{Lemont}
  \state{IL}
  \country{USA}
}

\author{Hongwei Jin}
\email{jinh@anl.gov}
\affiliation{%
  \institution{Argonne National Laboratory}
  \city{Lemont}
  \state{Illinois}
  \country{USA}
}

\author{Kibaek Kim}
\email{kimk@anl.gov}
\affiliation{%
  \institution{Argonne National Laboratory}
  \city{Lemont}
  \state{IL}
  \country{USA}
}

\author{Liang Zhao}
\email{liang.zhao@emory.edu}
\affiliation{%
  \institution{Emory University}
  \city{Atlanta}
  \state{GA}
  \country{USA}
}


\begin{abstract}
AC Optimal Power Flow (ACOPF) and Security-Constrained Unit Commitment (SCUC) are fundamental optimization problems in power system operations. ACOPF serves as the physical backbone of grid simulation and real-time operation, enforcing nonlinear power flow feasibility and network limits, while SCUC represents a core market-level decision process that schedules generation under operational and security constraints. Although these problems share the same underlying transmission network and physical laws, they differ in decision variables and temporal coupling, and prior learning-based approaches address them in isolation, resulting in disjoint models and representations.
We propose a learning framework that jointly models ACOPF and SCUC through a shared graph-based backbone that captures grid topology and physical interactions, coupled with task-specific decoders for static and temporal decision-making. Training includes solver supervision with physics-informed objectives to enforce AC feasibility and inter-temporal operational constraints.  
To evaluate generalization, we assess cross-case transfer on unseen grid topologies for ACOPF and SCUC without retraining, and systematic generalization on the UC-ACOPF problem using unsupervised, physics-based objectives and power dispatch consensus mechanism. Experiments across multiple grid scales demonstrate improved performance and transferability relative to existing learning-based baselines, indicating that the model can support learning across heterogeneous power system optimization problems.
\end{abstract}

\maketitle

\section{Introduction}
Power system operation requires solving a diverse set of large-scale optimization problems over a shared physical transmission network. These problems differ fundamentally in mathematical structure, spanning continuous and discrete decision spaces, static and temporal formulations, and linear approximations as well as highly nonlinear, nonconvex physical constraints. From an optimization perspective, this landscape includes both nonlinear optimizations driven by AC power flow physics and large mixed-integer programs governed by inter-temporal operational requirements, making power grid optimization a representative example of heterogeneous, structured decision-making under strict reliability and latency constraints. Two representative problems that capture this diversity are AC Optimal Power Flow (ACOPF) and Security-Constrained Unit Commitment (SCUC). ACOPF is a static, continuous optimization problem that computes AC-feasible operating points by enforcing nonlinear power flow equations and network limits, while SCUC is a large-scale, multi-period mixed-integer problem that determines which generators to switch on/off and power dispatch decisions over time with minimum generation cost, typically relying on simplified or linearized network models. Despite these differences, both problems are defined on the same grid topology and governed by identical generator--bus relationships and electrical laws, which shape the set of feasible and economically meaningful solutions. In operational practice, such problems are solved repeatedly under tight time constraints, and near-optimal or approximate solutions are routinely preferred over exact ones. Even SCUC itself relies on simplified representations of network physics, reflecting a long-standing trade-off between fidelity and tractability in real-world grid operation~\cite{pandey2023large}. While classical solvers provide strong guarantees, their computational cost can limit applicability at scale, motivating learning-based surrogates that aim to deliver fast, high-quality solutions while respecting physical constraints~\cite{pan2020deepopf}.

Classical optimization solvers form the backbone of power system operation, offering strong feasibility and optimality guarantees for both SCUC and ACOPF. In practice, SCUC is solved using large-scale mixed-integer optimization techniques such as branch-and-bound or branch-and-cut with decomposition and relaxation strategies, while ACOPF is addressed using nonlinear programming methods including interior-point and sequential quadratic programming approaches~\cite{zimmerman1997matpower,frank2012optimal,bienstock2014chance}. Although highly effective, these solvers can incur substantial computational cost for large networks, long planning horizons, or tightly coupled formulations, limiting their applicability in time-critical or large-scale settings.
Motivated by these limitations, learning-based approaches have demonstrated strong performance in accelerating both ACOPF and SCUC, including direct surrogate models~\cite{pan2020deepopf,piloto2024canos,li2025constraint}, learning-guided heuristics, and warm-start methods for classical solvers~\cite{diehl2019warm}. However, most existing approaches address these problems in isolation, training separate models tailored to each task. ACOPF models typically emphasize static, bus-level feasibility and nonlinear power flow physics~\cite{pan2020deepopf,piloto2024canos,trigui2025graph}, whereas SCUC models prioritize temporal decision-making and inter-temporal operational constraints, often under simplified or linearized network representations~\cite{ramesh2023spatio,liu2025multi}. As a result, learned representations re-encode the same grid structure independently and specialize to individual objectives. This fragmentation limits reuse and transfer, making it difficult to adapt learned models to coupled or intermediate problem formulations without retraining from scratch.

These limitations become most pronounced in coupled and hybrid optimization settings, where multiple classes of constraints and decision layers must be satisfied simultaneously. Such problems arise naturally in power system operation, where discrete commitment decisions, nonlinear network physics, and temporal operational constraints interact in nontrivial ways. A representative example is the UC-ACOPF problem, which enforces consistency between unit commitment schedules and AC-feasible power flow at each time step~\cite{gomez2025relax}. Broadly UC-ACOPF exemplifies a growing class of grid optimization problems that combine mixed-integer and nonlinear constraints across time, posing challenges that are substantially more complex than those encountered in isolated formulations.
From a learning perspective, these coupled settings expose a central challenge of systematic generalization: models must generalize not only across instances, but across problem formulations defined by different combinations of constraints, objectives, and decision variables. High-quality supervised labels require repeatedly solving large-scale mixed-integer nonlinear programs, which is computationally expensive and often infeasible at scale~\cite{zhang2023solving}. More importantly, models trained separately for individual tasks lack a shared representation through which heterogeneous decisions—such as commitment and dispatch—can be made when coupled. As a result, learning-based methods do not generalize to new combinations of constraints, tighter coupling regimes, or unseen grid topologies, and extending them to such settings typically requires retraining or additional supervision. This motivates learning grid representations that capture physical structure and remain transferable across diverse optimization problems, a perspective aligned with emerging work on jointly learning related optimization problems and constraint-conditioned models~\cite{drakulic2024goal,hamann2024foundation,cai2025multi}.

To address these challenges, we propose a learning framework that jointly models ACOPF and SCUC through a shared graph-based spatial representation of the power grid. The core idea is to learn a common encoder that captures grid topology and physical interactions, while using task-specific decoders to account for the static nature of ACOPF and the temporal structure of SCUC. Training combines supervision from solver solutions with physics-informed penalties, encouraging the shared representation to remain consistent with both nonlinear AC feasibility and inter-temporal operational constraints.
To assess generalization, we evaluate the jointly trained encoder along two complementary axes. First, we measure cross-case generalization in ACOPF and SCUC by evaluating on unseen grid topologies without retraining. Second, to assess systematic generalization under tighter coupling, we evaluate the shared encoder to the UC-ACOPF setting and adapt only lightweight task-specific components using unsupervised, physics-based objectives that reconcile commitment and dispatch decisions.
\\
{\textbf{Core Contributions. }}
\begin{itemize}[leftmargin=6pt, topsep=3pt]
\item \textbf{Joint modeling of ACOPF and SCUC}  
We propose a learning framework that jointly models ACOPF and SCUC using a shared graph-based backbone that captures grid topology and physical interactions across static and temporal optimization tasks.
\item \textbf{Systematic generalization to UC-ACOPF via unsupervised adaptation.}  
We show that the pre-trained shared encoder can be reused under tighter coupling by adapting only lightweight task-specific components to the UC-ACOPF problem using unsupervised, physics-based objectives, without retraining the encoder.
\item \textbf{Empirical evaluation across tasks and grid scales.}  
We empirically validate the proposed approach on multiple power grid topologies, demonstrating improved performance and transferability on ACOPF, SCUC, and UC-ACOPF relative to baselines.
\end{itemize}
\section{Related Work}
\subsection{Learning-Based Surrogate and Hybrid Models for OPF and SCUC}
Learning-based approaches to optimal power flow (OPF), particularly ACOPF, are commonly categorized into direct surrogate models and hybrid learning–optimization pipelines. 
Direct surrogates learn mappings from system conditions to primal OPF variables (e.g., voltages and generator dispatch), offering large inference speedups over classical solvers but often suffering from feasibility and constraint-violation issues when deployed end-to-end~\cite{pan2020deepopf}.
Recent work mitigates these limitations by embedding physical laws and constraints directly into training, rather than relying solely on supervised solutions~\cite{nellikkath2022physics,varbella2024physics}.
In parallel, hybrid learning-to-optimize methods retain an iterative solver in the loop, using neural models to predict high-quality initializations or intermediate quantities that accelerate convergence while preserving robustness and feasibility guarantees~\cite{xie2025neural}.
Reinforcement learning has also been explored for real-time ACOPF, framing power dispatch objectives as reward signals~\cite{feng2024safe}. For security-constrained unit commitment (SCUC), learning is primarily applied to commitment or dispatch approximation, constraint or network reduction, and warm-starting MILP solvers.
Owing to SCUC’s spatio-temporal structure, prior work combines graph-based spatial encoders with recurrent or transformer-based temporal models to capture ramping and minimum up/down constraints, enabling scalable reduced formulations and fast inference~\cite{ramesh2023spatio,liu2025multi}.
Reinforcement learning–based relaxation methods further target real-time SCUC by trading exact optimality for reduced solution times~\cite{zhu2025reinforcement}.

\subsection{Graph Neural Architectures for Power Grid Problems}
Graph neural networks (GNNs) have become a dominant representation choice for power-grid learning because they align naturally with grid topology and enable inductive biases that scale across operating points and, in some settings, across networks. Early OPF learning largely used homogeneous message passing (e.g., GCN/GAT/MPNN) over bus–branch graphs~\cite{mahto2024gat,suri2025powergnn}, while more recent work emphasizes typed/heterogeneous graphs that explicitly model buses, generators, loads, and multiple edge relations, improving expressivity and interpretability for grid components and controls~\cite{ghamizi2024opf,trigui2025graph,lopez2024optimal}. Architectural variations further include line-graph and dual-graph constructions to explicitly model branch-level interactions~\cite{ebtia2024power,ying2019gnnexplainer}, as well as hierarchical or multi-scale designs that combine local message passing with global aggregation to capture long-range electrical effects~\cite{bianchi2023expressive,zhao2025power}. Attention-based graph transformers extend classical MPNNs by allowing adaptive, long-range interactions beyond fixed-hop neighborhoods~\cite{arowolo2025towards,puny2020global}, while spatio-temporal graph architectures integrate temporal attention or recurrent modules on top of spatial graph encoders to handle time-coupled dynamics~\cite{moradpour2025spatio}. Collectively, these architectural developments reflect a shift from uniform, shallow message passing toward structured, heterogeneous, and spatio-temporal graph representations tailored to the physical and operational characteristics.

\subsection{Multi Task Learning for Optimization Problems}
Multi-task learning and shared-representation training have recently shown strong potential in optimization and decision-making domains. Prior work has explored learning reusable computation or solution structure across tasks, including graph neural networks trained to imitate classical algorithms and sequence-modeling approaches that learn from expert trajectories across multiple tasks~\cite{ibarz2022generalist, stooke2021decoupling}. More recent efforts in optimization propose generalist or foundation-style models which employ a shared backbone with task-specific adapters to solve multiple combinatorial optimization problems and enable transfer across tasks~\cite{drakulic2024goal, reed2022generalist}. Related lines of work on MILP representation learning also reinforces this that a common latent space can be learned to support learning across different optimization formulations~\cite{cai2025multi}. Recent discussions on grid foundation models further support the view that a shared backbone can be leveraged across multiple power grid optimization tasks, but focuses mostly on masked pre-training~\cite{hamann2024foundation, puech2024optimal}.

\section{Methodology}
\subsection{Problem Setup and Preliminaries}
\label{sec:problem_setup}
Power-grid operational tasks form a heterogeneous family of optimization problems defined over a shared physical transmission network, differing in decision variables and constraint structure.
We model the power grid as a typed graph $G=(V,E)$, where nodes represent buses, generators, loads, and shunt elements, and edges represent transmission lines with associated electrical parameters.
Each generator $g \in \mathcal{G} \subset V$ is connected to a designated bus, while loads and shunt elements are attached to bus nodes through typed relations.
This graph structure encodes the physical connectivity and interaction patterns governing power system operation and is shared across all tasks considered in this work.

Decision variables are associated with entities in $G$ and depend on the operational task.
In the AC Optimal Power Flow (ACOPF) problem, decision variables are defined primarily at the bus and generator levels: for each bus, ACOPF optimizes voltage magnitudes and phase angles, while for each generator it determines real and reactive power injections.
These variables are constrained by nonlinear AC power flow equations, voltage magnitude bounds, and transmission line thermal limits, yielding a continuous, static optimization problem.
In contrast, the Security-Constrained Unit Commitment (SCUC) problem introduces temporal and discrete decision-making.
Over a planning horizon of length $T$, SCUC optimizes binary unit commitment indicators and real power dispatch levels for each generator, subject to inter-temporal operational constraints such as ramping limits, minimum up/down time requirements, and generation capacity bounds.
Network constraints are typically enforced using linearized approximations of AC power flow, resulting in a mixed-integer, temporally coupled optimization problem defined on the same graph $G$.

The coupled UC--ACOPF setting requires consistency between unit commitment decisions and AC-feasible power flow solutions at each time step.
Here, commitment decisions constrain admissible dispatch levels, while AC power flow constraints must be satisfied conditional on commitment, yielding a mixed-integer nonlinear optimization problem that is substantially more challenging than either ACOPF or SCUC in isolation.
The detailed mathematical formulations for these problems are given in Appendix~\ref{app:formulations}.

Let $x_{\mathrm{opf}}^\star(G)$ and $x_{\mathrm{uc}}^\star(G)$ denote the respective optimal solutions of ACOPF and SCUC.
Our objective is to learn a shared graph-based encoder that maps the grid structure and features into a common latent representation, from which task-specific decoders can approximate these solution mappings while remaining compatible with the physical and operational constraints of the power system.

\subsection{Shared Spatial Encoder with Task-Specific Decoders}
ACOPF and SCUC are share the same power-grid graph but differ in their input features, decision variables, and constraint structures. To exploit their shared spatial structure while accommodating task-specific information, we construct a common model consisting of task-specific feature encoders, a shared projection layer, and a heterogeneous graph transformer. This produces unified spatial embeddings that support task-specific decoding for both problems. Let $G$ denote the power-grid graph defined in Section~\ref{sec:problem_setup}, and $\mathbf{x}^{(k)}$ denote the task-specific node and edge features associated with task $k\in\{\mathrm{OPF},\mathrm{UC}\}$. Task-specific features are first processed by lightweight MLP encoders and then mapped into a common latent space through a shared projection
\begin{equation}
\label{eq:projection}
\tilde{\mathbf{x}}^{(k)} = \mathrm{Proj}\!\left(\mathbf{x}^{(k)}\right),
\qquad
\tilde{\mathbf{x}}^{(k)} \in \mathbb{R}^{d_0},
\end{equation}
where $\mathrm{Proj}(\cdot)$ is a learnable, task-agnostic mapping shared across ACOPF and SCUC. The projected features are processed by a heterogeneous graph transformer parameterized by $\phi$, yielding node-level spatial representations
\begin{equation}
\label{eq:hetero}
\mathbf{H}_{\mathrm{sp}} = \mathrm{HGT}_{\phi}\!\left(G,\tilde{\mathbf{x}}^{(k)}\right),
\qquad
\mathbf{H}_{\mathrm{sp}} \in \mathbb{R}^{|V|\times d}.
\end{equation}
The projection and transformer parameters are shared across tasks and optimized to capture spatial regularities of the power grid that are invariant to the specific optimization objective.

\paragraph{ACOPF decoding.}
ACOPF predictions are derived from the spatial embeddings as it is static problem. For each bus node $i\in V$, the corresponding embedding $\mathbf{h}^{\mathrm{sp}}_i$ is decoded to voltage magnitude and phase angle,
\begin{equation}
(|V_i|,\theta_i) = \mathrm{MLP}_{\mathrm{opf}}^{\mathrm{bus}}(\mathbf{h}^{\mathrm{sp}}_i).
\end{equation}
For generator nodes $g\in\mathcal{G}$, the same spatial embeddings are decoded to real and reactive power injections,
\begin{equation}
(P_g,Q_g) = \mathrm{MLP}_{\mathrm{opf}}^{\mathrm{gen}}(\mathbf{h}^{\mathrm{sp}}_g).
\end{equation}
These quantities are constrained through the AC power flow equations and network limits, ensuring physical consistency of the OPF solution.

\paragraph{SCUC decoding.}
SCUC introduces a temporal dimension while reusing the same spatial encoder. For each time step $t\in\{1,\dots,T\}$, node embeddings are augmented with time-dependent information prior to temporal reasoning. Specifically, for each node $v\in V$, the input token at time $t$ is constructed as
\begin{equation}
\mathbf{z}^{(0)}_{t,v}
=
\left[
\mathbf{h}^{\mathrm{sp}}_v \;\middle\|\;
\ell_{t,v} \;\middle\|\;
\mathbf{e}_t
\right],
\end{equation}
where $\ell_{t,v}$ denotes the load value at node $v$ and time $t$, and $\mathbf{e}_t$ is a learnable time embedding. For bus nodes, $\ell_{t,v}$ corresponds to the active power demand, while for generator and other node types $\ell_{t,v}=0$. This augmentation injects time-varying demand information directly into the node-level representations. Let $\mathbf{Z}^{(0)}_{t} \in \mathbb{R}^{|V|\times d'}$ denote the matrix obtained by stacking $\mathbf{z}^{(0)}_{t,v}$ over all nodes $v \in V$.
The augmented tokens are projected back to the embedding dimension using a shared projection,
\begin{equation}
\mathbf{H}^{t}_{\mathrm{sp}} = \mathrm{Proj}_{\mathrm{time}}\!\left(\mathbf{Z}^{(0)}_t\right),
\qquad
\mathbf{H}^{t}_{\mathrm{sp}} \in \mathbb{R}^{|V|\times d},
\end{equation}
and the resulting sequence $\{\mathbf{H}^{t}_{\mathrm{sp}}\}_{t=1}^{T}$ is processed by a temporal transformer
\begin{equation}
\mathbf{S}_{\mathrm{UC}} =
\mathrm{TempTransformer}_{\psi}\!\left(\{\mathbf{H}^{t}_{\mathrm{sp}}\}_{t=1}^{T}\right),
\end{equation}
which enables information exchange across nodes and time steps through self-attention.
For each generator node $g\in\mathcal{G}$, the corresponding time-indexed embedding $\mathbf{s}^{\mathrm{UC}}_{t,g}$ is decoded into commitment and dispatch predictions,
\begin{equation}
u_{t,g} = \mathrm{MLP}_{\mathrm{on/off}}(\mathbf{s}^{\mathrm{UC}}_{t,g}),
\qquad
\tilde{p}_{t,g} = \mathrm{MLP}_{\mathrm{dispatch}}(\mathbf{s}^{\mathrm{UC}}_{t,g}),
\end{equation}
where $u_{t,g}$ denotes the unit commitment decision and $\tilde{p}_{t,g}$ is the corresponding real power dispatch proposal.
This ensures that time-varying load information is injected at the token level before temporal attention, allowing generator representations to incorporate demand signals through attention-based message passing, while preserving a shared spatial representation across tasks.

\subsection{Training Objectives and Update Configuration}
\label{sec:loss_and_updates}
Training is performed by supervision from optimal solver solutions together with physics-informed penalties that promote feasibility with respect to power-system constraints. The shared spatial encoder and task-specific decoders are optimized jointly during this stage.

\paragraph{Supervised objectives.}
For ACOPF, predictions are defined over bus and generator nodes. For each bus $i\in V_{\mathrm{bus}}$, the model predicts voltage magnitude and phase angle $(|V_i|,\theta_i)$, while for each generator $g\in\mathcal{G}$ it predicts real and reactive power injections $(P_g,Q_g)$. Given optimal solver outputs $(|V_i|^\star,\theta_i^\star,P_g^\star,Q_g^\star)$, we define a mean-squared error loss
\begin{equation}
\mathcal{L}^{\mathrm{sup}}_{\mathrm{opf}}
=
\sum_{i\in V_{\mathrm{bus}}}
\left\|
(|V_i|,\theta_i) - (|V_i|^\star,\theta_i^\star)
\right\|_2^2
+
\sum_{g\in\mathcal{G}}
\left\|
(P_g,Q_g) - (P_g^\star,Q_g^\star)
\right\|_2^2.
\end{equation}

Similarly for SCUC, predictions are defined on generator nodes across a planning horizon of length $T$. For each generator $g\in\mathcal{G}$ and time step $t$, the model predicts a commitment decision $u_{t,g}$ and a real power dispatch $\tilde{p}_{t,g}$. Given optimal solutions $\{u^\star_{t,g},p^\star_{t,g}\}$, we define
\begin{equation}
\mathcal{L}^{\mathrm{sup}}_{\mathrm{uc}}
=
\sum_{t=1}^{T}\sum_{g\in\mathcal{G}}
\left[
\mathrm{BCE}(u_{t,g},u^\star_{t,g})
+
\left\|
\tilde{p}_{t,g} - p^\star_{t,g}
\right\|_2^2
\right],
\end{equation}
where $\mathrm{BCE}(\cdot,\cdot)$ denotes the binary cross-entropy loss.

\paragraph{Physics-informed losses for ACOPF and SCUC}
To promote AC feasibility, we penalize violations of nodal power balance and thermal line limits induced by the predicted ACOPF variables. Let $\Delta P_i$ and $\Delta Q_i$ denote the active and reactive power mismatches at bus $i$, and let $S_\ell$ denote the apparent power flow magnitude on transmission line $\ell$ with limit $\overline{S}_\ell$. We define
\begin{equation}
\mathcal{L}^{\mathrm{phys}}_{\mathrm{opf}}
=
\sum_{i \in V_{\mathrm{bus}}}
\left(
|\Delta P_i| + |\Delta Q_i|
\right)
\;+\;
\sum_{\ell \in E_{\mathrm{line}}}
\max\!\left(0,\; S_\ell - \overline{S}_\ell \right).
\end{equation}
SCUC feasibility is encouraged through penalties on inter-temporal and operational constraints. Let $R_g^{\uparrow}$ and $R_g^{\downarrow}$ denote ramp-up and ramp-down limits for generator $g$, and let $P_g^{\min}$ and $P_g^{\max}$ denote its capacity limits. Using the predicted dispatch $\tilde{p}_{t,g}$, we define
\begin{align}
\mathcal{L}^{\mathrm{phys}}_{\mathrm{uc}}
&=
\sum_{t=2}^{T}\sum_{g\in\mathcal{G}}
\Big[
\max(0,\tilde{p}_{t,g}-\tilde{p}_{t-1,g}-R_g^{\uparrow})^2
+
\max(0,\tilde{p}_{t-1,g}-\tilde{p}_{t,g}-R_g^{\downarrow})^2
\Big]
\nonumber\\
&\quad+
\sum_{t=1}^{T}\sum_{g\in\mathcal{G}}
\Big[
\max(0,\tilde{p}_{t,g}-u_{t,g}P_g^{\max})^2
+
\max(0,u_{t,g}P_g^{\min}-\tilde{p}_{t,g})^2
\Big].
\label{eq:ref_01}
\end{align}

The final training objective is given by
\begin{equation}
\mathcal{L}_{\mathrm{total}}
=
\alpha\,\mathcal{L}^{\mathrm{sup}}_{\mathrm{opf}}
+
\beta\,\mathcal{L}^{\mathrm{sup}}_{\mathrm{uc}}
+
\gamma\,\mathcal{L}^{\mathrm{phys}}_{\mathrm{opf}}
+
\delta\,\mathcal{L}^{\mathrm{phys}}_{\mathrm{uc}},
\end{equation}
where $\alpha,\beta,\gamma,\delta$ are hyperparameters. As this is a multi-constrained optimization, we adopt GradNorm~\cite{chen2018gradnorm} technique for shared encoder weight updates across tasks to balance gradient magnitudes across heterogeneous task losses.

\subsection{Unsupervised Consensus-Based UC--ACOPF Fine-Tuning}
\label{sec:consensus}
After pre-training on ACOPF and SCUC, we fine-tune for the UC-ACOPF problem in a fully unsupervised manner. In this stage, we \emph{freeze} the shared spatial encoder (and all intermediate representation layers) and update \emph{only} the task-specific decoder heads. The supervision signal comes solely from (i) physics-violation penalties and (ii) a consensus loss that aligns generator real-power predictions across the two heads i.e. ACOPF and SCUC decoders.

Let $\mathbf{H}_{\mathrm{sp}}=\mathrm{HGT}_{\phi}(G,\mathbf{x})$ denote frozen spatial embeddings, where $\phi$ is fixed. The SCUC head predicts commitment and a dispatch proposal
\begin{equation}
(\hat{u}_{t,g},\,\tilde{p}^{\mathrm{UC}}_{t,g})=\mathrm{D}_{\mathrm{UC}}(\mathbf{H}_{\mathrm{sp}}),
\qquad \hat{u}_{t,g}\in(0,1),
\end{equation}
and the ACOPF head produces AC-feasible quantities at each time step from which we extract a generator real-power dispatch
\begin{equation}
\hat{p}^{\mathrm{AC}}_{t,g}=\mathrm{ExtractPg}\!\left(\mathrm{D}_{\mathrm{OPF}}(\mathbf{H}_{\mathrm{sp}})\right).
\end{equation}
We enforce hard coupling between UC and OPF by defining the \emph{effective} dispatch used throughout constraint evaluation as
\begin{equation}
p_{t,g} \;=\; \hat{u}_{t,g}\,\hat{p}^{\mathrm{AC}}_{t,g}.
\label{eq:hard_coupling_pg}
\end{equation}
This guarantees that predicted offline units contribute zero generation. We impose agreement between the SCUC dispatch proposal and the coupled ACOPF dispatch via consensus loss
\begin{equation}
\mathcal{L}_{\mathrm{cons}}
=
\sum_{t=1}^{T}\sum_{g\in\mathcal{G}}
\left\|
\tilde{p}^{\mathrm{UC}}_{t,g}-p_{t,g}
\right\|_2^2.
\label{eq:consensus_pg}
\end{equation}
Physics-informed penalties are computed using the coupled dispatch in~\eqref{eq:hard_coupling_pg} and ramping and capacity violations are penalized similarly as described in Equation~\eqref{eq:ref_01}.
For ACOPF, we penalize AC feasibility violations (e.g., nodal power-balance mismatches and line thermal limit exceedances) induced by the OPF head outputs, denoted compactly as $\mathcal{L}^{\mathrm{phys}}_{\mathrm{opf}}$. The final unsupervised fine-tuning objective for UC-ACOPF is
\begin{equation}
\mathcal{L}_{\mathrm{UC\text{-}ACOPF}}
=
\mathcal{L}_{\mathrm{cons}}
+
\lambda_{\mathrm{opf}}\,\mathcal{L}^{\mathrm{phys}}_{\mathrm{opf}}
+
\lambda_{\mathrm{uc}}\,\mathcal{L}^{\mathrm{phys}}_{\mathrm{uc}},
\end{equation}
and gradients update only the decoder parameters while keeping $\phi$ fixed.

\begin{algorithm}[t]
\caption{Unsupervised consensus-based UC--ACOPF fine-tuning}
\label{alg:consensus_finetune}
\centering
\resizebox{\columnwidth}{!}{%
\begin{minipage}{\columnwidth}
\small
\DontPrintSemicolon
\SetKwInOut{Input}{Input}
\SetKwInOut{Output}{Output}

\Input{
Typed grid graph $G=(V,E)$ with features $\mathbf{x}$;
planning horizon $T$;
frozen shared encoder $\mathrm{HGT}_{\phi}$;
ACOPF decoder $\mathrm{D}_{\mathrm{OPF}}$;
SCUC decoder $\mathrm{D}_{\mathrm{UC}}$;
loss weights $\lambda_{\mathrm{opf}}, \lambda_{\mathrm{uc}}$;
learning rate $\rho$.
}
\Output{
Updated decoder parameters $(\Theta_{\mathrm{OPF}},\Theta_{\mathrm{UC}})$.
}

\BlankLine
\For{each fine-tuning step}{
    \tcp{Frozen spatial encoding}
    $\mathbf{H}_{\mathrm{sp}} \leftarrow \mathrm{HGT}_{\phi}(G,\mathbf{x})$\;

    \BlankLine
    \tcp{SCUC head: commitment and dispatch proposal}
    $(\hat{u}_{t,g},\tilde{p}^{\mathrm{UC}}_{t,g}) \leftarrow 
    \mathrm{D}_{\mathrm{UC}}(\mathbf{H}_{\mathrm{sp}}),
    \quad \forall t,g$\;

    \BlankLine
    \tcp{ACOPF head: AC-feasible generator dispatch}
    $\hat{p}^{\mathrm{AC}}_{t,g} \leftarrow 
    \mathrm{ExtractPg}(\mathrm{D}_{\mathrm{OPF}}(\mathbf{H}_{\mathrm{sp}})),
    \quad \forall t,g$\;

    \BlankLine
    \tcp{Hard UC--OPF coupling}
    $p_{t,g} \leftarrow \hat{u}_{t,g}\,\hat{p}^{\mathrm{AC}}_{t,g},
    \quad \forall t,g$\;

    \BlankLine
    \tcp{Consensus loss}
    $\mathcal{L}_{\mathrm{cons}} \leftarrow
    \sum_{t=1}^{T}\sum_{g\in\mathcal{G}}
    \left\|
    \tilde{p}^{\mathrm{UC}}_{t,g} - p_{t,g}
    \right\|_2^2$\;

    \BlankLine
    \tcp{Physics-informed penalties (coupled dispatch)}
    $\mathcal{L}^{\mathrm{phys}}_{\mathrm{uc}} \leftarrow
    \mathrm{UCPhysPenalty}(p_{t,g},\hat{u}_{t,g})$\;

    $\mathcal{L}^{\mathrm{phys}}_{\mathrm{opf}} \leftarrow
    \mathrm{ACPhysPenalty}(\mathrm{D}_{\mathrm{OPF}}(\mathbf{H}_{\mathrm{sp}}))$\;

    \BlankLine
    \tcp{Final unsupervised objective}
    $\mathcal{L}_{\mathrm{UC\text{-}ACOPF}} \leftarrow
    \mathcal{L}_{\mathrm{cons}}
    + \lambda_{\mathrm{opf}}\,\mathcal{L}^{\mathrm{phys}}_{\mathrm{opf}}
    + \lambda_{\mathrm{uc}}\,\mathcal{L}^{\mathrm{phys}}_{\mathrm{uc}}$\;

    \BlankLine
    \tcp{Decoder-only update (encoder frozen)}
    $(\Theta_{\mathrm{OPF}},\Theta_{\mathrm{UC}}) \leftarrow
    (\Theta_{\mathrm{OPF}},\Theta_{\mathrm{UC}})
    - \rho\,\nabla \mathcal{L}_{\mathrm{UC\text{-}ACOPF}}$\;
}
\end{minipage}
}
\end{algorithm}

\subsection{Theoretical Analysis}
\label{sec:theory}

We provide a theoretical analysis for the unsupervised
consensus-based UC--ACOPF fine-tuning procedure.
Our analysis focuses on the structural correctness of the UC-ACOPF coupling
and on the role of the fine-tuning objective as a penalty-based relaxation
of the coupled UC-ACOPF constraints.

\paragraph{UC--OPF coupling.}
Recall the effective coupled dispatch defined in
Eq.~\eqref{eq:hard_coupling_pg},
\(
p_{t,g} := \hat{u}_{t,g}\,\hat{p}^{\mathrm{AC}}_{t,g}
\).
\begin{lemma}[Shutdown consistency]
\label{lem:shutdown_main}
If $\hat{u}_{t,g}=0$, then $p_{t,g}=0$ regardless of
$\hat{p}^{\mathrm{AC}}_{t,g}$.
Moreover, if $0\le \hat{p}^{\mathrm{AC}}_{t,g}\le P_g^{\max}$ and
$\hat{u}_{t,g}\in[0,1]$, then $p_{t,g}\le \hat{u}_{t,g}P_g^{\max}$.
\end{lemma}

Lemma~\ref{lem:shutdown_main} shows that the proposed coupling eliminates
spurious generation from offline units while preserving generator
capacity limits by construction.
The relaxation $\hat{u}_{t,g}\in(0,1)$ enables differentiable optimization
and is used only during training.

\paragraph{Penalty-based fine-tuning objective.}
Let $\Theta$ denote the decoder parameters and define the consensus
residual $r(\Theta)$ together with the UC and OPF constraint residuals
$c_{\mathrm{uc}}(\Theta)$ and $c_{\mathrm{opf}}(\Theta)$.
The unsupervised fine-tuning objective is
\begin{align}
\mathcal{L}_{\mathrm{UC\text{-}ACOPF}}(\Theta)
&=
\|r(\Theta)\|_2^2
+ \lambda_{\mathrm{uc}}
\sum_i \bigl[\max(0,c_{\mathrm{uc},i}(\Theta))\bigr]^2
\nonumber\\
&\quad
+ \lambda_{\mathrm{opf}}
\sum_j \bigl[\max(0,c_{\mathrm{opf},j}(\Theta))\bigr]^2 .
\label{eq:penalty_main}
\end{align}
\paragraph{Optimization endpoint.}
Let $\Theta^\star$ denote the parameters obtained at convergence of the
unsupervised fine-tuning procedure, i.e., a point satisfying
\begin{equation}
\|\partial_\Theta
\mathcal{L}_{\mathrm{UC\text{-}ACOPF}}(\Theta^\star)\|
\le \eta
\end{equation}
for some tolerance $\eta>0$.
\begin{assumption}[Bounded constraint gradients]
\label{assump:bounded_grad_main}
The UC and OPF constraint residuals are differentiable and there exists a
constant $L_c>0$ such that
\begin{equation}
\|\nabla_\Theta c(\Theta)\| \le L_c
\end{equation}
for all $\Theta$ in a neighborhood of $\Theta^\star$, where
$c(\Theta)$ denotes the concatenation of
$c_{\mathrm{uc}}(\Theta)$ and $c_{\mathrm{opf}}(\Theta)$.
\end{assumption}
\begin{theorem}[Near-stationarity implies approximate feasibility]
\label{thm:feasibility_main}
Under Assumption~\ref{assump:bounded_grad_main}, if
$\|\partial_\Theta
\mathcal{L}_{\mathrm{UC\text{-}ACOPF}}(\Theta^\star)\|\le \eta$,
then the coupled solution
$(\hat{u}(\Theta^\star),p(\Theta^\star))$
satisfies
\begin{equation}
\sum_i [\max(0,c_{\mathrm{uc},i}(\Theta^\star))]^2
+
\sum_j [\max(0,c_{\mathrm{opf},j}(\Theta^\star))]^2
=
\mathcal{O}\!\left(
\frac{\eta^2}{\min(\lambda_{\mathrm{uc}},\lambda_{\mathrm{opf}})}
\right).
\end{equation}
\end{theorem}

\paragraph{Interpretation.}
Theorem~\ref{thm:feasibility_main} shows that decoder-only fine-tuning
drives the model toward approximate feasibility of a relaxed UC--ACOPF
problem.
The consensus term enforces agreement between UC and OPF dispatch
predictions, while penalty terms reduce violations of inter-temporal UC
constraints and AC power-flow constraints.

\section{Experiments}
\paragraph{Datasets Used}
For ACOPF problem, we used OPFData~\cite{lovett2024opfdata}, a large-scale dataset of ACOPF solutions generated using industry-standard solvers under diverse load perturbration. OPFData provides representative power network topologies, each paired with 300K feasible operating points obtained by perturbing load profiles.
For the SCUC task, we use standard benchmark unit commitment instances provided by the UnitCommitment.jl framework. These instances are derived from widely used MATPOWER test networks and originate from the Power Systems Test Case Archive maintained by the University of Washington~\cite{coffrin2014nesta}. The benchmarks follow established SCUC formulations from the literature and span 365 days, with a 36-hour planning horizon for each network topology. They include detailed generator characteristics, transmission network constraints, reserve requirements, and time-coupled operational constraints such as ramping limits and minimum up/down times.
Each SCUC instance is solved offline using an industry-grade mixed-integer linear programming (MILP) solver (CPLEX) to obtain high-quality reference solutions. The resulting commitment schedules, dispatch trajectories, and constraint-related variables are stored in a structured graph-based representation and used for training and evaluation of learning-based models. We conduct our evaluations on the Case-14, Case-30, Case-57, and Case-118 networks. Additional implementation and dataset details including their features are provided in the Appendix~\ref{app:data}.

\paragraph{Baselines.}
For ACOPF, we compare against CANOS~\cite{piloto2024canos} which is a GNN based approach, graph attention network (GAT)–based power flow model~\cite{le2025dpfaga} and Heterogeneous Graph Transformer. 
For SCUC, we compare against STModel~\cite{ramesh2023spatio}, which combines graph convolutional networks with LSTM-based temporal modeling, and MSTT~\cite{liu2025multi}, a multi-scale spatio-temporal transformer.
To the best of our knowledge, no existing learning-based method jointly models ACOPF and SCUC within a unified architecture or learns representations shared across these heterogeneous optimization tasks. Accordingly, we evaluate multi-task learning performance through controlled ablation studies that compare single-task training against joint training with shared encoders, isolating the effect of representation sharing. This evaluation protocol enables a fair assessment of whether shared model improve performance and generalization beyond task-specific baselines.

\paragraph{Evaluation Metrics.}
We evaluate ACOPF and SCUC using metrics that capture prediction accuracy, physical feasibility, cost optimality, and inference efficiency. 
For ACOPF, prediction accuracy is measured using mean squared error (MSE) at the bus and generator levels (MSE$_\text{Bus}$, MSE$_\text{Gen}$). 
Physical feasibility is assessed via the normalized power balance RMSE (PF Viol.), which measures deviations from the AC nodal power balance equations. 
We also report total AC constraint violation norm (Viol. Norm), which aggregates violations of power balance and transmission line flow limits and is normalized by $\sqrt{\mathcal{N}}$, where $\mathcal{N}$ denotes the number of buses, to reduce sensitivity to network size. 
Cost performance is quantified using the optimality gap relative to the oracle solver-optimal ACOPF cost(Opt. Gap). For SCUC, unit commitment accuracy (Acc) measures correct on/off decisions across generators and time steps, and dispatch quality is evaluated using RMSE on active power generation $P_g$. Constraint satisfaction is quantified by the percentage of violated UC constraints (\% Viol.), including power balance, ramping, and capacity limits, and economic performance is measured using the optimality gap relative to the MILP oracle solution.
Inference efficiency is evaluated by average per-instance inference time under identical settings and compared against classical solver with speedups as the ratio of solver runtime to model inference time.

\subsection{Performance on SCUC and ACOPF}
\begin{table}[t]
\centering
\setlength{\tabcolsep}{4pt}
\renewcommand{\arraystretch}{1.2}
\caption{SCUC performance comparison on \textbf{case-118}. Percentage changes are computed relative to the strongest SCUC-only baseline for each metric.}
\label{tab:scuc_case118}
\resizebox{\columnwidth}{!}{
\begin{tabular}{l|cccc}
\toprule
\textbf{Method}
& Acc $\uparrow$
& RMSE($P_g$) $\downarrow$
& \% Viol. $\downarrow$
& Opt. Gap $\downarrow$ \\
\midrule

STModel
& 85.19\% & 0.28 & 0.50\% & 10.72\% \\

MSTT
& 80.40\% & 0.16 & \textbf{0.03\%} & 11.98\% \\

\midrule
\rowcolor{blue!12}
\textbf{Ours}
& \textbf{88.88\%} ($\uparrow$4.33\%)
& \textbf{0.11} ($\downarrow$31.3\%)
& 0.07\% ($\uparrow$133.3\%)
& \textbf{8.86\%} ($\downarrow$17.3\%) \\

\bottomrule
\end{tabular}
}

\footnotesize{
$\uparrow$ / $\downarrow$ indicate whether higher or lower values are better.
}
\label{tab:scuc}
\end{table}

\begin{table}[t]
\centering
\setlength{\tabcolsep}{4pt}
\renewcommand{\arraystretch}{1.2}
\caption{ACOPF performance comparison on \textbf{case-118}. Percentage changes are computed relative to the strongest ACOPF-only baseline for each metric.}
\label{tab:acopf_case118}
\resizebox{\columnwidth}{!}{
\begin{tabular}{l|cccc}
\toprule
\textbf{Method}
& MSE$_\text{Bus}$ $\downarrow$
& MSE$_\text{Gen}$ $\downarrow$
& PF Viol. $\downarrow$
& Opt. Gap $\downarrow$ \\
\midrule

CANOS
& 0.010 & 0.030 & 0.25 & 0.90\% \\

GAT
& 0.005 & \textbf{0.020} & 0.15 & 2.30\% \\

HGT
& 0.003 & \textbf{0.020} & \textbf{0.04} & 1.90\% \\

\midrule
\rowcolor{blue!12}
\textbf{Ours}
& \textbf{0.002} ($\downarrow$33.3\%)
& \textbf{0.020} (0\%)
& 0.09 ($\uparrow$125\%)
& \textbf{0.80\%} (0\%) \\

\bottomrule
\end{tabular}
}
\footnotesize{
$\uparrow$ / $\downarrow$ indicate whether higher or lower values are better.
}
\label{tab:acopf}
\end{table}
Table~\ref{tab:scuc} and~\ref{tab:acopf} evaluate whether a shared spatial encoder trained jointly on ACOPF and SCUC improves task performance on case-118 against task specific baselines. On SCUC, the shared encoder improves commitment accuracy from 85.19\% to 88.88\% while reducing dispatch RMSE by 31.3\% and optimality gap by 17.3\%. These gains occur despite using the same temporal decoder class as prior work, indicating that improvements originate from the spatial encoder rather than temporal modeling capacity.
Constraint violations increase from 0.03\% to 0.07\% relative to the most conservative baseline, but remain below 0.1\%, showing that joint training trades a small feasibility margin for substantially improved accuracy and cost.
On ACOPF, the shared encoder reduces bus-level MSE from 0.003 to 0.002 while matching the best generator-level MSE (0.020) and achieving the lowest optimality gap (0.80\%). Power-flow violations increase from 0.04 to 0.09, but remain significantly lower than homogeneous GNN baselines. This again reflects a consistent tradeoff: joint training improves predictive and cost accuracy while slightly relaxing feasibility margins, without destabilizing solutions.

\begin{figure}
    \centering
    \includegraphics[width=0.99\linewidth]{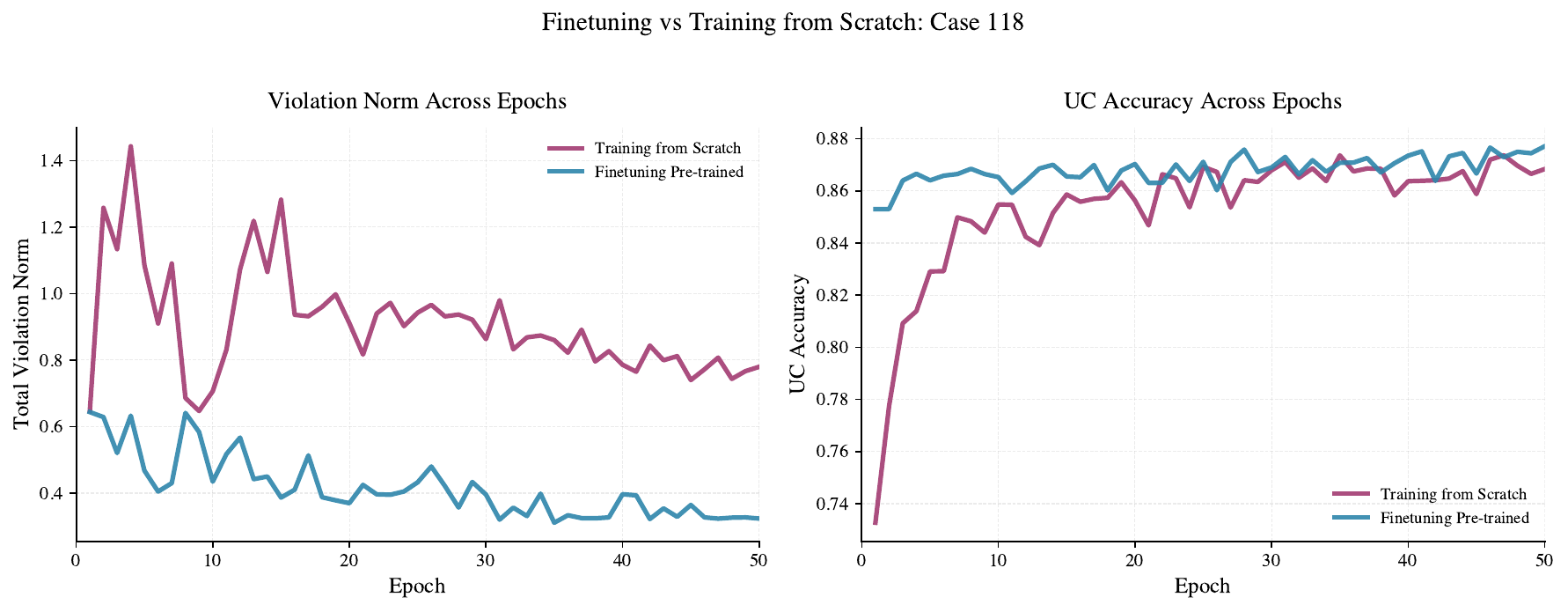}
    \caption{Fine-tuning vs. training from scratch on case-118}
    \label{fig:finetune_vs_scratch_case118}
\end{figure}

\subsection{Cross Case Generalization}
\begin{table*}[t]
\centering
\setlength{\tabcolsep}{3pt}
\renewcommand{\arraystretch}{1.2}
\caption{Cross-case zero shot generalization performance of the proposed shared encoder model. Models are trained on smaller grid cases and evaluated on unseen larger cases.}
\label{tab:cross_case_generalization}
\resizebox{0.99\textwidth}{!}{
\begin{tabular}{c|cccc|ccccc}
\toprule
\multirow{2}{*}{\textbf{Train $\rightarrow$ Test}} 
& \multicolumn{4}{c|}{\textbf{SCUC}} 
& \multicolumn{5}{c}{\textbf{ACOPF}} \\
\cmidrule(lr){2-5}\cmidrule(lr){6-10}
& Acc $\uparrow$ 
& RMSE($P_g$) $\downarrow$ 
& \%Viol. $\downarrow$ 
& Opt. Gap $\downarrow$
& MSE\_BUS $\downarrow$ 
& MSE\_Gen $\downarrow$ 
& PF Viol. $\downarrow$ 
& Viol. Norm $\downarrow$
& Opt. Gap $\downarrow$ \\
\midrule
14 $\rightarrow$ 30 
& \textbf{84.28}\% 
& 0.00367 
& 0.11\% 
& +1.72\% 
& 0.06673 
& 0.04041 
& 0.34 
& 1.251
& -2.98\% \\

14,30 $\rightarrow$ 57 
& \textbf{83.52}\% 
& 0.00814 
& 0.10\% 
& -3.47\% 
& 0.04098 
& 1.03860 
& 0.27 
& 1.602
& -4.44\% \\

14,30,57 $\rightarrow$ 118 
& \textbf{81.80}\% 
& 0.00287 
& 0.53\% 
& +0.61\% 
& 0.00635 
& 0.04431 
& 0.38 
& 2.356
& +0.26\% \\
\bottomrule
\end{tabular}
}
\label{tab:zero_shot_transfer}
\end{table*}
\begin{figure}
    \centering
    \includegraphics[width=0.8\linewidth]{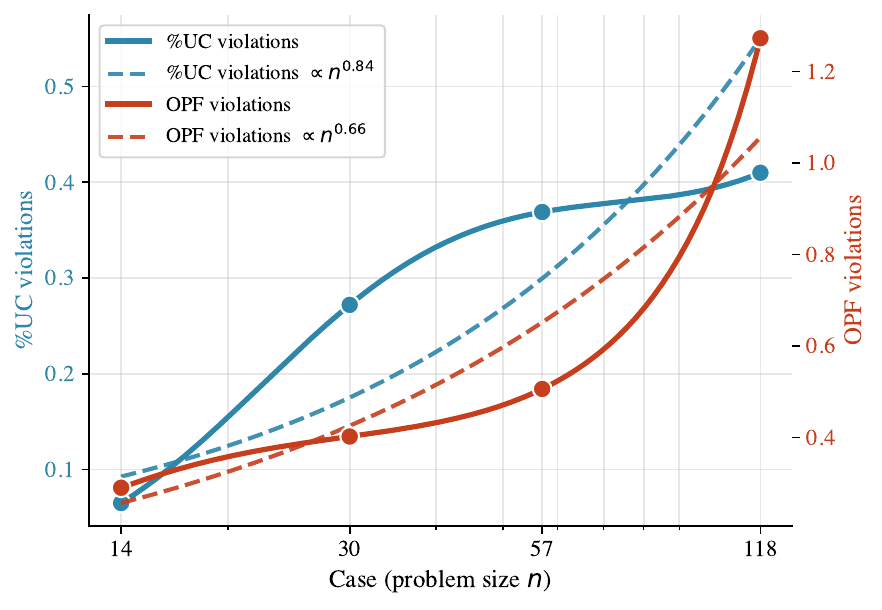}
    \caption{Scaling of constraint violations with grid size. UC and ACOPF violation rates are shown across increasing network sizes (case-14 to case-118), with dashed lines indicating fitted power-law trends.}
    \label{fig:scaling_violations}
\end{figure}
Table~\ref{tab:zero_shot_transfer} evaluates whether the shared model captures grid-invariant structure by testing zero-shot transfer to larger, unseen networks. When trained on case-14 and evaluated on case-30, SCUC accuracy reaches 84.28\% with only 0.11\% constraint violations. As the test grid size increases, performance degrades slightly, achieving 83.52\% accuracy on case-57 and 81.80\% on case-118. Importantly, constraint violations remain bounded below 0.6\%, while optimality gaps stay within $\pm$1\% across all transfer settings.
A similar trend is observed for ACOPF. As additional grid instances are incorporated during training, bus-level mean squared error decreases monotonically, reaching 0.00635 on case-118. This behavior indicates that exposure to multiple network topologies improves representation robustness. Overall, these results reinforce that the shared encoder learns transferable representations that generalize well across network sizes and topologies.
Figure~\ref{fig:scaling_violations} illustrates further how constraint violations scale with problem size. Both UC and ACOPF violations increase smoothly as network size grows from case-14 to case-118, following approximate power-law trends. UC violations exhibit a steeper scaling behavior
($\sim n^{0.84}$) than ACOPF violations ($\sim n^{0.66}$), consistent with the stronger sensitivity of unit commitment to network expansion.
Importantly power-grid feasibility problems such as AC power flow are known to be NP-hard in the worst case under general formulations~\cite{bienstock2019strong}, making controlled scaling behavior nontrivial in practice. In this context, the observed sublinear growth of constraint violations
provides evidence that the learned representation captures salient grid
structure and helps mitigate practical scaling effects.
Figure~\ref{fig:finetune_vs_scratch_case118} further examines cross-case transfer by comparing fine-tuning from pretrained encoders (trained on smaller grids) against training from scratch on the target case.
On case-118, fine-tuning yields substantially lower violation norms throughout training, converging to below 0.35 within 10 epochs, whereas training from scratch exhibits large initial violations exceeding 1.4 and requires over 40 epochs to stabilize near 0.8.
This gap persists across the entire optimization trajectory, indicating that representations learned on smaller grids provide a strong feasibility-aware initialization for larger networks. In terms of commitment accuracy, fine-tuning starts at 85.3\% and improves steadily to 87.8\%, while training from scratch begins at 73.2\% and converges more slowly to a comparable but slightly lower plateau. The faster convergence and reduced variance under fine-tuning supports that the shared encoder transfers structurally meaningful spatial features.

\subsection{Performance on UC-ACOPF Task}
\begin{table}[t]
\centering
\setlength{\tabcolsep}{4pt}
\renewcommand{\arraystretch}{1.2}
\caption{Feasibility and optimality comparison between training from scratch and fine-tuning with a shared encoder on UCACOPF Task.}
\label{tab:quality_scratch_vs_finetune}
\resizebox{\columnwidth}{!}{
\begin{tabular}{c|c|cccc}
\toprule
\textbf{Case}
& \textbf{Setting}
& \textbf{PF Viol.} $\downarrow$
& \textbf{RMSE($P_g$)} $\downarrow$
& \textbf{\% UC Viol.} $\downarrow$
& \textbf{Cost Gap} $\downarrow$ \\
\midrule

\multirow{2}{*}{14}
& Scratch (50 ep)
& 0.73 & 0.25 & 1.8 & 3.9 \\
& \textbf{Finetuned (10 ep)}
& \textbf{0.01} & \textbf{0.05} & \textbf{0.1} & 4.9 \\

\midrule

\multirow{2}{*}{30}
& Scratch (50 ep)
& 0.89 & 0.33 & 0.9 & 6.11 \\
& \textbf{Finetuned (10 ep)}
& \textbf{0.09} & \textbf{0.11} & \textbf{0.9} & \textbf{4.5} \\

\midrule

\multirow{2}{*}{118}
& Scratch (50 ep)
& 1.55 & 0.83 & 1.9 & 5.2 \\
& \textbf{Finetuned (10 ep)}
& \textbf{0.18} & \textbf{0.15} & \textbf{1.3} & \textbf{4.9} \\

\bottomrule
\end{tabular}
}

\footnotesize{
Lower is better for all reported metrics.
}
\label{tab:finetune_vs_scratch}
\end{table}
\begin{figure}
    \centering
    \includegraphics[width=0.8\linewidth]{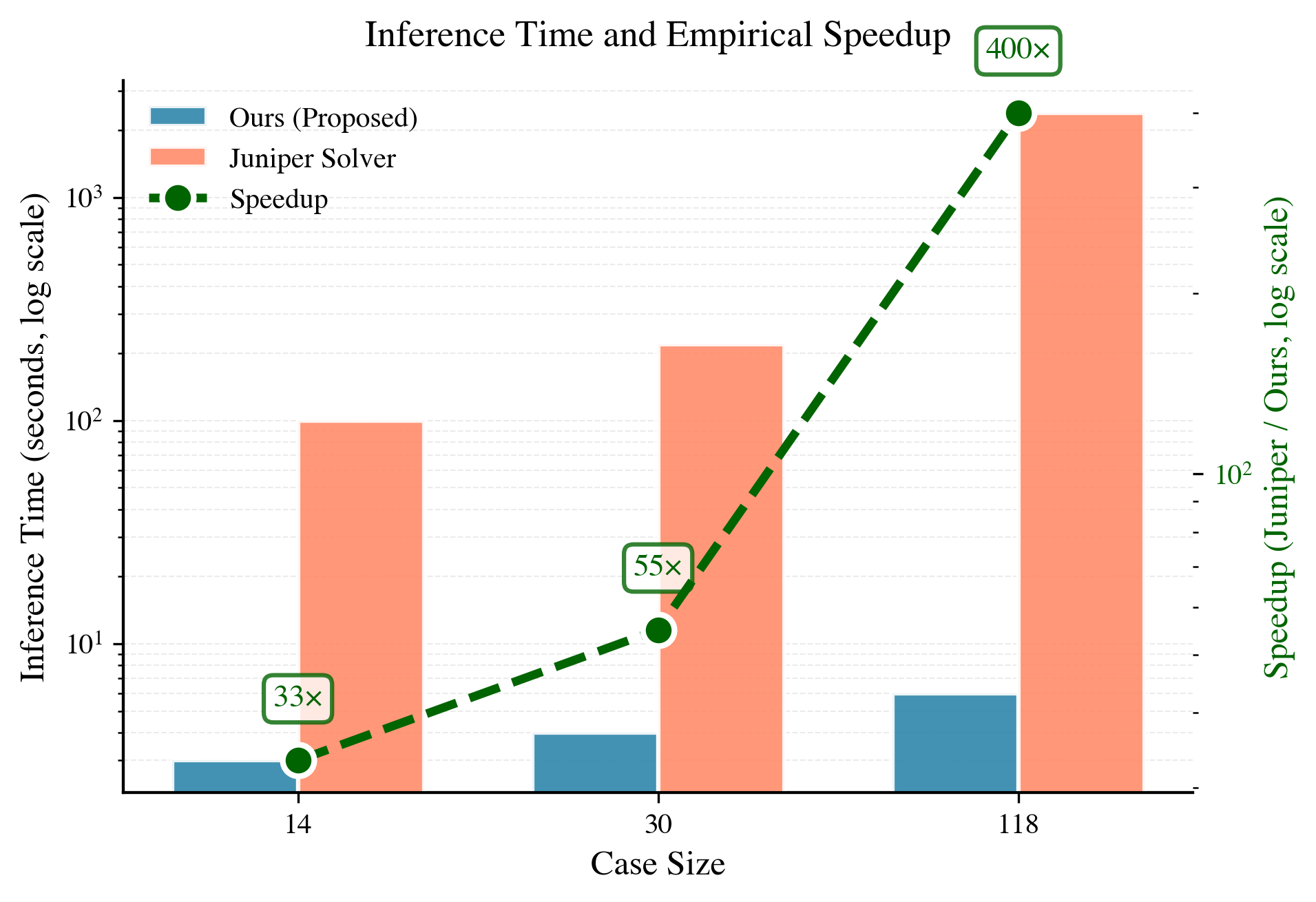}
    \caption{Inference time comparison between the proposed model and the Juniper(oracle solver).}
    \label{fig:inference_speed}
\end{figure}
Table~\ref{tab:finetune_vs_scratch} compares training from scratch against fine-tuning from the shared encoder on the UC--ACOPF task. Fine-tuning consistently yields substantial feasibility improvements across all grid sizes. On case-118, power-flow violations are reduced from 1.55 to 0.18, while dispatch mean squared error decreases from 0.83 to 0.15, despite using $5\times$ fewer training epochs.
Because the shared encoder is frozen during fine-tuning, these improvements can be attributed to the compatibility of the learned representations with both unit commitment and AC power flow constraints. The results indicate that joint ACOPF-SCUC training induces spatial embeddings that already encode much of the structure required for the coupled problem, thereby reducing the burden on task-specific decoders. Figure~\ref{fig:inference_speed} shows inference time comparisons against the Juniper solver for the coupled UC-ACOPF task.
Across all cases, the proposed model achieves two to three orders of magnitude speedup, with per-instance inference times of 3--6 seconds compared to 100--2400 seconds for Juniper, which is oracle solver. The gap widens with network size: on case-118, inference is $400\times$ faster, reflecting the superlinear scaling of mixed-integer nonlinear solvers relative to the fixed-depth neural inference. These results highlight that the proposed approach is not intended to replace oracle solvers in terms of exact optimality, but to provide fast and near-optimal solutions suitable for large-scale or time-critical settings.

\subsection{Ablation Studies}
\label{sec:ablation}

\begin{table}[t]
\centering
\setlength{\tabcolsep}{4pt}
\renewcommand{\arraystretch}{1.2}
\caption{Ablation results on \textbf{case-118} evaluating the contribution of shared projection and heterogeneous modeling.}
\label{tab:ablation}
\resizebox{\columnwidth}{!}{
\begin{tabular}{l|ccccc}
\toprule
\textbf{Variant}
& Acc $\uparrow$
& \% UC Viol. $\downarrow$
& PF Viol. $\downarrow$
& Opt. Gap (SCUC) $\downarrow$
& Opt. Gap (ACOPF) $\downarrow$ \\
\midrule

\rowcolor{blue!12}
\textbf{Ours}
& \textbf{88.88}
& \textbf{0.07}
& \textbf{0.09}
& \textbf{8.86}
& \textbf{0.80} \\

\midrule
- shared projection
& 86.83
& 0.21
& 0.10
& 8.84
& 4.11 \\

- heterogeneous modeling
& 85.40
& 0.50
& 0.29
& 9.12
& 2.20 \\

\bottomrule
\end{tabular}
}
\footnotesize{
Lower is better for all metrics except accuracy. "-" indicates without that ablated that component.
}
\end{table}
Table~\ref{tab:ablation} quantifies the contribution of shared projection and heterogeneous modeling to joint SCUC--ACOPF performance.
Removing the shared projection~\eqref{eq:projection} reduces SCUC accuracy from 88.88\% to 86.83\% and increases UC violations from 0.07\% to 0.21\%, while leaving the SCUC optimality gap nearly unchanged (8.86\% vs.\ 8.84\%).
In contrast, ACOPF optimality degrades substantially, with the optimality gap increasing from 0.80\% to 4.11\%.
This indicates that naive feature pooling across tasks produces representations that remain competitive for SCUC but are misaligned with ACOPF feasibility and cost structure, highlighting the role of the shared projection in enforcing cross-task representational compatibility.
Removing heterogeneous model~\eqref{eq:hetero} with homogeneous one causes a broader and more uniform degradation.
SCUC accuracy drops to 85.40\%, UC violations increase to 0.50\%, and power-flow violations rise from 0.09 to 0.29.
Both SCUC and ACOPF optimality gaps worsen (9.12\% and 2.20\%, respectively), indicating reduced ability to capture generator-bus interactions and network constraints when relational types are ignored. Overall, the ablations show that improvements from joint training do not arise from parameter sharing alone and effective transfer across SCUC and ACOPF requires both a shared projection to align task representations and heterogeneous modeling.

\section{Conclusion}
In this paper, we proposed a joint learning framework for ACOPF and SCUC that leverages a shared graph-based encoder to capture grid topology and physical interactions across static and temporal optimization problems. By training a common spatial representation with task-specific decoders, the model enables transfer across grid topologies and reuse on complex problem which couples these two tasks. We demonstrated that the pretrained encoder can be adapted to the UC–ACOPF setting using unsupervised, physics-informed objectives without retraining, yielding improved feasibility and faster convergence compared to training from scratch. These results indicate that learning shared, physically grounded representations is an effective approach for systematic generalization across heterogeneous power grid optimization problems.

\section{Limitations and Ethical Considerations}
This work focuses on a family of power grid optimization problems that share a common network structure and physical laws, and the proposed generalization results should be interpreted within this setting. While the framework demonstrates effective transfer across ACOPF, SCUC, and their coupled formulation UC-ACOPF, extending the approach to problems with substantially different constraint structures or operational objectives can be interesting future direction.
From an ethical perspective, this work focuses on methodological advances in optimization and does not introduce new privacy or societal risks; however, any operational use of learning-based surrogates in power systems should be accompanied by appropriate safeguards, validation, and human oversight to mitigate risks arising from approximation errors.
\begin{acks}
This material is based upon work supported by Laboratory Directed Research and Development (LDRD) funding from Argonne National Laboratory, provided by the Director, Office of Science, of the U.S. Department of Energy under contract DE-AC02-06CH11357.

An award of computer time was provided by the ASCR Leadership Computing Challenge (ALCC) program. This research used resources of the Argonne Leadership Computing Facility, which is a U.S. Department of Energy Office of Science User Facility operated under contract DE-AC02-06CH11357.

This research used resources of the National Energy Research Scientific Computing Center (NERSC), a Department of Energy User Facility using NERSC award ALCC-ERCAP0038201.

The submitted manuscript has been created by UChicago Argonne, LLC, Operator of Argonne National Laboratory (``Argonne”). Argonne, a U.S. Department of Energy Office of Science laboratory, is operated under Contract No. DE-AC02-06CH11357. The U.S. Government retains for itself, and others acting on its behalf, a paid-up nonexclusive, irrevocable worldwide license in said article to reproduce, prepare derivative works, distribute copies to the public, and perform publicly and display publicly, by or on behalf of the Government. The Department of Energy will provide public access to these results of federally sponsored research in accordance with the DOE Public Access Plan (http://energy.gov/downloads/doe-public-access-plan).
\end{acks}

\bibliographystyle{ACM-Reference-Format}
\bibliography{references}

\appendix
\section{Theoretical Details}
\label{app:theory}

This appendix provides additional theoretical details supporting
Section~\ref{sec:theory}.
We present proofs of the feasibility result, clarify the treatment of
nonsmooth penalty terms, and formalize a notion of consensus-stationarity
for the proposed unsupervised UC--ACOPF fine-tuning procedure.

\subsection{Proof of Theorem~\ref{thm:feasibility_main}}
\label{app:proof_feasibility}

We restate the objective for convenience:
\begin{equation}
\label{eq:appendix_penalty}
\begin{aligned}
\mathcal{L}_{\mathrm{UC\text{-}ACOPF}}(\Theta)
&= \|r(\Theta)\|_2^2 + \lambda_{\mathrm{uc}}
\sum_i \bigl[\max(0,c_{\mathrm{uc},i}(\Theta))\bigr]^2 \\
&\quad + \lambda_{\mathrm{opf}}
\sum_j \bigl[\max(0,c_{\mathrm{opf},j}(\Theta))\bigr]^2 .
\end{aligned}
\end{equation}

where $r(\Theta)$ denotes the consensus residual,
and $c_{\mathrm{uc}}(\Theta)$ and $c_{\mathrm{opf}}(\Theta)$ denote the
UC and ACOPF constraint residual vectors evaluated on the coupled
dispatch $p(\Theta)$.

\paragraph{Subgradient expression.}
Since hinge penalties $\max(0,\cdot)$ are nonsmooth, we work with Clarke
subgradients.
A subgradient of~\eqref{eq:appendix_penalty} can be written as
\begin{align}
\partial_\Theta \mathcal{L}
&=
2\nabla_\Theta r(\Theta)^\top r(\Theta)
\nonumber\\
&\quad+
2\lambda_{\mathrm{uc}}
\nabla_\Theta c_{\mathrm{uc}}(\Theta)^\top
\max(0,c_{\mathrm{uc}}(\Theta))
+
2\lambda_{\mathrm{opf}}
\nabla_\Theta c_{\mathrm{opf}}(\Theta)^\top
\max(0,c_{\mathrm{opf}}(\Theta)).
\label{eq:appendix_subgrad}
\end{align}

\paragraph{Bounding constraint violations.}
Let $\Theta^\star$ satisfy $\|\partial_\Theta \mathcal{L}(\Theta^\star)\|\le \eta$.
From the subgradient expression~\eqref{eq:appendix_subgrad} and the triangle
inequality, we obtain:
\begin{align}
\eta
&\ge
\Biggl\|
\begin{aligned}
&2\lambda_{\mathrm{uc}}\,
\nabla_\Theta c_{\mathrm{uc}}(\Theta^\star)^\top
\max\!\bigl(0,c_{\mathrm{uc}}(\Theta^\star)\bigr) \\
&\quad+
2\lambda_{\mathrm{opf}}\,
\nabla_\Theta c_{\mathrm{opf}}(\Theta^\star)^\top
\max\!\bigl(0,c_{\mathrm{opf}}(\Theta^\star)\bigr)
\end{aligned}
\Biggr\|
\nonumber\\
&\ge
2\min(\lambda_{\mathrm{uc}},\lambda_{\mathrm{opf}})
\Bigl\|
\nabla_\Theta c(\Theta^\star)^\top
\max\!\bigl(0,c(\Theta^\star)\bigr)
\Bigr\|
\nonumber\\
&\quad-
2\|\nabla_\Theta r(\Theta^\star)\|\,
\|r(\Theta^\star)\| .
\label{eq:eta_lower_bd}
\end{align}
Now, by Assumption~\ref{assump:bounded_grad_main}, $\|\nabla_\Theta c(\Theta^\star)\|\le L_c$,
hence
\begin{align}
\Bigl\|
\nabla_\Theta c(\Theta^\star)^\top
\max\!\bigl(0,c(\Theta^\star)\bigr)
\Bigr\|
&\le
\|\nabla_\Theta c(\Theta^\star)\|\;
\|\max\!\bigl(0,c(\Theta^\star)\bigr)\|
\nonumber\\
&\le
L_c\;
\|\max\!\bigl(0,c(\Theta^\star)\bigr)\| .
\label{eq:jac_bd}
\end{align}

If we additionally assume $\|\nabla_\Theta r(\Theta)\|\le L_r$ locally
(and either $\|r(\Theta^\star)\|$ is small in practice), then combining~\eqref{eq:eta_lower_bd}--\eqref{eq:jac_bd}
yields
\begin{equation}
\|\max(0,c(\Theta^\star))\|
\le
\mathcal{O}\!\left(
\frac{\eta + L_r\|r(\Theta^\star)\|}{\min(\lambda_{\mathrm{uc}},\lambda_{\mathrm{opf}})}
\right).
\end{equation}
When the consensus residual is small at convergence, the bound reduces to
\begin{equation}
\|\max(0,c(\Theta^\star))\|
\le
\mathcal{O}\!\left(
\frac{\eta}{\min(\lambda_{\mathrm{uc}},\lambda_{\mathrm{opf}})}
\right),
\end{equation}
and squaring both sides gives the feasibility bound in
Theorem~\ref{thm:feasibility_main}.
\hfill$\square$

\subsection{Nonsmooth Analysis and Subgradients}
\label{app:nonsmooth}

The UC--ACOPF fine-tuning objective includes hinge penalties of the form
$\max(0,\cdot)$, which are nonsmooth but locally Lipschitz.
All gradient-based statements in Section~\ref{sec:theory} and this
appendix are therefore interpreted in terms of Clarke subgradients.
This is standard in the analysis of nonconvex penalty methods and holds
almost everywhere.
See, e.g., \citet{clarke1990optimization} for formal definitions.

\subsection{Consensus-Stationarity}
\label{app:consensus_stationary}

We formalize a notion of approximate stationarity that captures the
interaction between consensus alignment and constraint penalties.

\begin{definition}[Consensus-stationary point]
\label{def:consensus_stationary}
A parameter vector $\Theta$ is an $(\epsilon,\eta)$-consensus-stationary
point if the following conditions hold:
\begin{enumerate}[leftmargin=*,label=(\roman*)]
\item \textbf{Consensus alignment:}
$\|r(\Theta)\|_2 \le \epsilon$.
\item \textbf{Approximate stationarity:}
There exist vectors $\mu_{\mathrm{uc}}\ge 0$ and
$\mu_{\mathrm{opf}}\ge 0$ such that
\[
\big\|
\nabla_\Theta r(\Theta)^\top \nu
+
\nabla_\Theta c_{\mathrm{uc}}(\Theta)^\top \mu_{\mathrm{uc}}
+
\nabla_\Theta c_{\mathrm{opf}}(\Theta)^\top \mu_{\mathrm{opf}}
\big\|
\le \eta,
\]
for some $\nu$ with $\|\nu\|_2 \le 2\epsilon$.
\item \textbf{Approximate complementarity:}
\[
\sum_i \mu_{\mathrm{uc},i}\max(0,c_{\mathrm{uc},i}(\Theta)) \le \eta,
\qquad
\sum_j \mu_{\mathrm{opf},j}\max(0,c_{\mathrm{opf},j}(\Theta)) \le \eta.
\]
\end{enumerate}
\end{definition}

\begin{theorem}[Consensus-stationarity from small loss and near-stationarity]
\label{thm:consensus_stationary}
Assume that $r(\Theta)$, $c_{\mathrm{uc}}(\Theta)$, and
$c_{\mathrm{opf}}(\Theta)$ are locally Lipschitz with bounded Jacobians.
If $\Theta^\star$ satisfies
\[
\|\partial_\Theta
\mathcal{L}_{\mathrm{UC\text{-}ACOPF}}(\Theta^\star)\| \le \eta,
\qquad
\|r(\Theta^\star)\|_2^2 \le \varepsilon,
\]
then $\Theta^\star$ is an
$(\sqrt{\varepsilon},\tilde{\eta})$-consensus-stationary point, where
$\tilde{\eta}$ depends continuously on
$\eta$, $\lambda_{\mathrm{uc}}$, and $\lambda_{\mathrm{opf}}$.
\end{theorem}

\begin{proof}
From $\|r(\Theta^\star)\|_2^2\le\varepsilon$, we have
$\|r(\Theta^\star)\|_2\le\sqrt{\varepsilon}$.
Using the subgradient expression~\eqref{eq:appendix_subgrad},
define
$\mu_{\mathrm{uc},i} =
2\lambda_{\mathrm{uc}}\max(0,c_{\mathrm{uc},i}(\Theta^\star))$
and analogously for $\mu_{\mathrm{opf}}$.
Setting $\nu := 2r(\Theta^\star)$ yields the stated approximate
stationarity and complementarity conditions.
\end{proof}

\paragraph{Interpretation.}
Consensus-stationarity formalizes the notion that the learned solution
approximately satisfies the Karush--Kuhn--Tucker conditions of a relaxed
coupled UC--ACOPF problem.
In practice, this provides a principled explanation for why minimizing
the unsupervised fine-tuning objective leads to UC--OPF consistency and
reduced physical constraint violations, even when training is restricted
to decoder parameters.

\section{Mathematical Formulations for Grid Optimization Problems}
\label{app:formulations}
This appendix summarizes standard mathematical formulations of the
AC Optimal Power Flow (ACOPF), Security-Constrained Unit Commitment (SCUC),
and the coupled UC--ACOPF problem. These formulations are provided for
reference and to clarify the optimization problems considered in the main text.

All problems are defined on a power grid represented as a typed graph
$G=(V,E)$ with bus set $\mathcal{N}=V_{\mathrm{bus}}$, generator set
$\mathcal{G}\subseteq V$, and transmission line set $\mathcal{L}=E$.
For each bus $i\in\mathcal{N}$, let $\mathcal{G}_i\subseteq\mathcal{G}$
denote the generators connected to bus $i$.
Active and reactive demands at bus $i$ are denoted by $(P_i^d,Q_i^d)$,
or $(P_{t,i}^d,Q_{t,i}^d)$ in the time-indexed setting.

\subsection{AC Optimal Power Flow (ACOPF)}
\label{app:acopf}

\paragraph{Decision variables.}
Voltage magnitudes and phase angles $(|V_i|,\theta_i)$ for each bus
$i\in\mathcal{N}$, and real and reactive power injections $(P_g,Q_g)$
for each generator $g\in\mathcal{G}$.

\paragraph{Objective.}
A standard quadratic generation cost is minimized:
\begin{equation}
\min \sum_{g\in\mathcal{G}}
\left(c_{2,g} P_g^2 + c_{1,g} P_g + c_{0,g}\right).
\end{equation}

\paragraph{AC power balance.}
Let $Y = G + jB$ denote the bus admittance matrix and
$\theta_{ij}=\theta_i-\theta_j$. For all buses $i\in\mathcal{N}$:
\begin{align}
\sum_{g\in\mathcal{G}_i} P_g - P_i^d &=
\sum_{j\in\mathcal{N}} |V_i||V_j|
\left(G_{ij}\cos\theta_{ij} + B_{ij}\sin\theta_{ij}\right), \\
\sum_{g\in\mathcal{G}_i} Q_g - Q_i^d &=
\sum_{j\in\mathcal{N}} |V_i||V_j|
\left(G_{ij}\sin\theta_{ij} - B_{ij}\cos\theta_{ij}\right).
\end{align}

\paragraph{Operational constraints.}
\begin{align}
P_g^{\min} \le P_g \le P_g^{\max}, \quad
Q_g^{\min} \le Q_g \le Q_g^{\max}, && \forall g\in\mathcal{G}, \\
V_i^{\min} \le |V_i| \le V_i^{\max}, && \forall i\in\mathcal{N}.
\end{align}

\paragraph{Thermal limits.}
For each transmission line $(i,j)\in\mathcal{L}$, let
$S_{ij}(|V|,\theta)$ denote the apparent power flow:
\begin{equation}
|S_{ij}(|V|,\theta)| \le S_{ij}^{\max},
\qquad
|S_{ji}(|V|,\theta)| \le S_{ij}^{\max}.
\end{equation}

\subsection{Security-Constrained Unit Commitment (SCUC)}
\label{app:scuc}

SCUC optimizes generator commitment and dispatch over a planning horizon
$t=1,\dots,T$.

\paragraph{Decision variables.}
Binary commitment variables $u_{t,g}\in\{0,1\}$ and real power dispatch
$p_{t,g}$ for each generator $g\in\mathcal{G}$ and time $t$.
Optional startup and shutdown variables are denoted by
$v_{t,g}$ and $w_{t,g}$.

\paragraph{Objective.}
\begin{equation}
\min \sum_{t=1}^{T}\sum_{g\in\mathcal{G}}
\left(c_g(p_{t,g}) + C_g^{\mathrm{su}} v_{t,g}
+ C_g^{\mathrm{sd}} w_{t,g}\right).
\end{equation}

\paragraph{Commitment logic.}
\begin{align}
u_{t,g} - u_{t-1,g} &= v_{t,g} - w_{t,g},
&& \forall g,\; t=2,\dots,T, \\
v_{t,g} + w_{t,g} &\le 1,
&& \forall g,\; t=2,\dots,T.
\end{align}

\paragraph{Capacity and ramping limits.}
\begin{align}
u_{t,g} P_g^{\min} \le p_{t,g} \le u_{t,g} P_g^{\max},
&& \forall g,\; t, \\
- R_g^{\downarrow} \le p_{t,g} - p_{t-1,g} \le R_g^{\uparrow},
&& \forall g,\; t=2,\dots,T.
\end{align}

\paragraph{Network balance (DC approximation).}
For each bus $i\in\mathcal{N}$ and time $t$:
\begin{equation}
\sum_{g\in\mathcal{G}_i} p_{t,g} - P_{t,i}^d
= \sum_{j:(i,j)\in\mathcal{L}} f_{t,ij},
\end{equation}
with line flows
$f_{t,ij} = B_{ij}(\delta_{t,i}-\delta_{t,j})$
and limits
$|f_{t,ij}| \le F_{ij}^{\max}$.

\subsection{UC-ACOPF}
\label{app:ucacopf}

The coupled UC--ACOPF problem enforces AC feasibility at each time step
while respecting unit commitment and inter-temporal operational constraints.

\paragraph{Decision variables.}
For each $t$:
bus voltages $(|V_{t,i}|,\theta_{t,i})$,
generator injections $(P_{t,g},Q_{t,g})$,
and commitment variables $u_{t,g}$.

\paragraph{Objective.}
\begin{equation}
\min \sum_{t=1}^{T}\sum_{g\in\mathcal{G}}
\left(c_g(P_{t,g}) + C_g^{\mathrm{su}} v_{t,g}
+ C_g^{\mathrm{sd}} w_{t,g}\right).
\end{equation}

\paragraph{Commitment--dispatch coupling.}
\begin{equation}
u_{t,g} P_g^{\min} \le P_{t,g} \le u_{t,g} P_g^{\max}, \qquad
u_{t,g} Q_g^{\min} \le Q_{t,g} \le u_{t,g} Q_g^{\max}.
\end{equation}

\paragraph{AC feasibility.}
For all buses $i\in\mathcal{N}$ and time steps $t$:
\begin{align}
\sum_{g\in\mathcal{G}_i} P_{t,g} - P_{t,i}^d &=
\sum_{j\in\mathcal{N}} |V_{t,i}||V_{t,j}|
\left(G_{ij}\cos(\theta_{t,i}-\theta_{t,j})
+ B_{ij}\sin(\theta_{t,i}-\theta_{t,j})\right), \\
\sum_{g\in\mathcal{G}_i} Q_{t,g} - Q_{t,i}^d &=
\sum_{j\in\mathcal{N}} |V_{t,i}||V_{t,j}|
\left(G_{ij}\sin(\theta_{t,i}-\theta_{t,j})
- B_{ij}\cos(\theta_{t,i}-\theta_{t,j})\right).
\end{align}

Voltage magnitude bounds and branch thermal limits are enforced at each
time step as in the ACOPF formulation.

\section{Implementation Details}
\label{app:data}
\subsubsection*{Model Architecture}
We use a heterogeneous graph transformer (HGT) as the shared encoder across all tasks. The encoder consists of 4 message-passing layers with hidden dimension 64 and 8 attention heads per layer. Residual connections and node-type--specific layer normalization are applied after each layer, with ReLU activations and dropout set to 0.1. The heterogeneous graph includes bus, generator, load, and shunt node types, with AC transmission lines and attachment links as edge types.

\subsubsection*{Feature Alignment}
ACOPF and SCUC expose task-dependent node features that are first mapped using task- and node-type--specific linear projections. A shared projection layer then aligns all node representations into a common 64-dimensional latent space, which serves as the input to the shared encoder.

\subsubsection*{Task-Specific Decoders}
The ACOPF decoder applies node-type--specific linear heads to the shared spatial embeddings to predict bus voltage magnitude and phase angle, and generator real and reactive power outputs.  
The SCUC decoder augments spatial embeddings with time-dependent load information and learnable temporal encodings, and processes them using a temporal Transformer with hidden dimension 128, 2 layers, and 4 attention heads over a 36-hour planning horizon. Generator-level outputs include unit-commitment logits and real-power dispatch trajectories.

\subsubsection*{UC--ACOPF Coupling}
For the coupled UC--ACOPF task, the shared encoder and feature-alignment layers are frozen and only task-specific decoders are updated. 
\subsubsection*{Optimization}
All models are trained using Adam with learning rate $5\times10^{-4}$, weight decay $1\times10^{-5}$, batch size 32, and $\ell_2$ gradient clipping at 1.0. During UC--ACOPF fine-tuning, the learning rate is reduced to $5\times10^{-5}$.

\subsubsection*{Implementation and Release}
Models are implemented in PyTorch using PyTorch Geometric for heterogeneous graph processing. Dataset preprocessing follows OPFData and UnitCommitment.jl specifications. Code and preprocessing scripts will be released upon paper acceptance.

\subsection{Dataset Details}
Tables~\ref{tab:scuc_features} and~\ref{tab:acopf_features} summarize the complete set of node-level and edge-level features used for each task, including their physical interpretation and dimensionality. ACOPF features primarily describe static electrical states and constraints required for AC feasibility, while SCUC features additionally include temporal information and operational constraints necessary for unit commitment and dispatch.
All models, including baselines, are trained and evaluated using the same feature representations to ensure fair comparison.

\paragraph{UC-ACOPF Data}
UC-ACOPF problem instances are constructed following standard power system benchmarking practice.
We start from canonical IEEE test networks (e.g., 9-bus, 30-bus, and 118-bus systems) and generate a 36-hour demand trajectory for each case.
The base bus-level demand specified in the original IEEE test case is scaled over time using historical ISO-NE hourly demand profiles (October 2018), yielding realistic temporal load variation.
To avoid trivially binding all generators and to induce meaningful unit commitment decisions, the resulting demand trajectory is further multiplied by a constant discount factor (0.7), ensuring that not all generators are forced online throughout the horizon.
This procedure follows common practice in UC-ACOPF benchmarking and is used solely to generate realistic evaluation scenarios~\cite{zhang2023solving}).

Because UC--ACOPF is a mixed-integer nonlinear program, solving it to optimality is computationally expensive.
Accordingly, oracle MINLP solutions are computed for only a small evaluation set (32 instances per network) using a state-of-the-art solver, and are used exclusively for quantitative evaluation.
Model fine-tuning on UC--ACOPF is performed in a fully unsupervised manner using physics-based objectives, without access to solver-generated labels.
\begin{table*}[t]
\centering
\small
\setlength{\tabcolsep}{4pt}
\renewcommand{\arraystretch}{1.2}
\caption{Feature and target specification for the SCUC task.}
\label{tab:scuc_features}
\resizebox{\linewidth}{!}{
\begin{tabular}{l|c|l|l|l|l}
\toprule
\textbf{Component} & \textbf{Idx} & \textbf{Name} & \textbf{Description} & \textbf{Units} & \textbf{Type} \\
\midrule
\multirow{7}{*}{Bus features (\texttt{bus.x})}
& 0 & base\_kv & Base voltage & kV & Float \\
& 1 & vmin & Minimum voltage magnitude & p.u. & Float \\
& 2 & vmax & Maximum voltage magnitude & p.u. & Float \\
& 3 & is\_pq & PQ bus indicator & 0/1 & Float \\
& 4 & is\_pv & PV bus indicator (unused) & 0/1 & Float \\
& 5 & is\_ref & Reference/slack bus indicator & 0/1 & Float \\
& 6 & is\_isolated & Isolated bus indicator & 0/1 & Float \\
\midrule
Bus load (\texttt{bus.load})
& -- & 36 steps & Load profile over horizon & MW & Temporal float \\
\midrule
\multirow{19}{*}{Generator features (\texttt{generator.x})}
& 0 & cost\_quadratic & Quadratic cost coefficient & \$/MW$^2$ & Float \\
& 1 & cost\_linear & Linear cost coefficient & \$/MW & Float \\
& 2 & cost\_offset & Constant cost coefficient & \$ & Float \\
& 3 & qmax & Maximum reactive power & MVar & Float \\
& 4 & qmin & Minimum reactive power & MVar & Float \\
& 5 & vg & Voltage setpoint & p.u. & Float \\
& 6 & mbase & Machine base power & MVA & Float \\
& 7 & pmax & Maximum active power & MW & Float \\
& 8 & pmin & Minimum active power & MW & Float \\
& 9 & ramp\_up\_limit & Ramp-up rate limit & MW/h & Float \\
& 10 & ramp\_down\_limit & Ramp-down rate limit & MW/h & Float \\
& 11 & startup\_limit & Startup power limit & MW & Float \\
& 12 & shutdown\_limit & Shutdown power limit & MW & Float \\
& 13 & min\_uptime & Minimum uptime & h & Float \\
& 14 & min\_downtime & Minimum downtime & h & Float \\
& 15 & initial\_status & Initial on/off duration & h & Float \\
& 16 & initial\_power & Initial power output & MW & Float \\
& 17 & pmin\_prod & Minimum production curve power & MW & Float \\
& 18 & pmax\_prod & Maximum production curve power & MW & Float \\
\midrule
\multirow{4}{*}{Generator targets (\texttt{generator.y})}
& 0 & is\_on & Unit commitment status (36 steps) & 0/1 & Float \\
& 1 & pg & Active power generation (36 steps) & MW & Float \\
& 2 & switch\_on & Startup indicator (36 steps) & 0/1 & Float \\
& 3 & switch\_off & Shutdown indicator (36 steps) & 0/1 & Float \\
\bottomrule
\end{tabular}
}
\end{table*}

\begin{table*}[t]
\centering
\tiny
\setlength{\tabcolsep}{3pt}
\renewcommand{\arraystretch}{1.15}
\caption{ACOPF feature and target specification.}
\label{tab:acopf_features}
\resizebox{\linewidth}{!}{
\begin{tabular}{l|l|l|l}
\toprule
\textbf{Component} & \textbf{Idx} & \textbf{Features} & \textbf{Units} \\
\midrule
Bus features (\texttt{bus.x})
& 0--6
& base\_kv (base V), vmin/vmax (V limits), is\_pq/is\_pv/is\_ref/is\_isolated (bus type)
& kV / p.u. / 0--1 \\
Bus targets (\texttt{bus.y})
& 0--1
& va (angle), vm (magnitude)
& rad / p.u. \\
Generator features (\texttt{generator.x})
& 0--10
& mbase (base), pg/qg (init. power), pmin/pmax, qmin/qmax (limits), vg (V set), cost$_2$/cost$_1$/cost$_0$ (cost)
& MVA / MW / MVar / p.u. / \$ \\
Generator targets (\texttt{generator.y})
& 0--1
& pg/qg (optimal power)
& MW / MVar \\
Load features (\texttt{load.x})
& 0--1
& pd/qd (load demand)
& p.u. \\
Shunt features (\texttt{shunt.x})
& 0--1
& bs/gs (shunt admittance)
& p.u. \\
AC line edges
& 0--8
& angmin/angmax (angle), r/x (impedance), b$_{fr}$/b$_{to}$ (charging), rate$_{a,b,c}$ (thermal)
& rad / p.u. / MVA \\
Transformer edges
& 0--10
& angmin/angmax, r/x, rate$_{a,b,c}$, tap, shift, b$_{fr}$/b$_{to}$
& rad / p.u. / MVA \\
\bottomrule
\end{tabular}
}
\end{table*}

\subsection{Additional Results}
\subsubsection*{Per-Network Evaluation of a Jointly Trained Model}
\begin{table*}[t]
\centering
\small
\setlength{\tabcolsep}{3pt}
\renewcommand{\arraystretch}{1.2}
\caption{Per-case feasibility and optimality metrics for SCUC and ACOPF evaluated using the model trained on all cases.}
\label{tab:per_case_metrics}
\resizebox{\textwidth}{!}{
\begin{tabular}{c|cccc|ccccc}
\toprule
\multirow{2}{*}{\textbf{Case}} 
& \multicolumn{4}{c|}{\textbf{SCUC}} 
& \multicolumn{5}{c}{\textbf{ACOPF}} \\
\cmidrule(lr){2-5}\cmidrule(lr){6-10}
& Acc $\uparrow$ 
& RMSE($P_g$) $\downarrow$ 
& \%Viol. $\downarrow$ 
& Opt. Gap $\downarrow$
& MSE$_{\text{Bus}}$ $\downarrow$
& MSE$_{\text{Gen}}$ $\downarrow$
& PF Viol. $\downarrow$
& Viol. Norm $\downarrow$
& Opt. Gap $\downarrow$ \\
\midrule
14
& 94.5\%
& 0.2398
& 0.065\%
& -2.8\%
& 0.0012
& 0.0059
& 0.453
& 0.2904
& +1.16\% \\

30
& 93.0\%
& 0.1224
& 0.272\%
& -2.2\%
& 0.0005
& 0.0033
& 0.558
& 0.4020
& +4.84\% \\

57
& 93.7\%
& 0.1922
& 0.369\%
& -7.5\%
& 0.0016
& 0.0263
& 0.675
& 0.5062
& +0.29\% \\

118
& 90.1\%
& 0.2223
& 0.410\%
& -0.6\%
& 0.0018
& 0.0247
& 2.102
& 1.2731
& +0.03\% \\
\bottomrule
\end{tabular}
}
\end{table*}
Table~\ref{tab:per_case_metrics} reports per-network performance for SCUC and ACOPF using a single jointly trained model evaluated across all cases. 
For SCUC, commitment accuracy remains high (90--95\%) across grid sizes, while UC constraint violations stay below 0.5\%, indicating stable enforcement of inter-temporal constraints as network complexity increases. 
Negative SCUC optimality gaps reflect discrepancies between the learned dispatch and the MILP oracle under linearized network models and are consistent with known relaxation effects in learning-based SCUC.
For ACOPF, both bus-level and generator-level MSEs remain low across all cases, demonstrating accurate recovery of continuous decision variables. 
Power-flow violations and total violation norms increase sub-linearly with network size, consistent with the scaling difficulty of AC feasibility, but remain bounded and grow smoothly rather than abruptly. ACOPF optimality gaps stay close to zero even on larger networks, indicating that economic performance is largely preserved despite modest feasibility degradation. 
Overall, these results show that a single shared encoder supports consistent accuracy and controlled feasibility behavior across heterogeneous grid sizes, aligning with the paper’s emphasis on systematic generalization rather than per-case specialization.


\end{document}